\documentclass{article}

\usepackage[nonatbib,preprint]{arxiv_2023}

\usepackage[utf8]{inputenc} 
\usepackage[T1]{fontenc}    
\usepackage[colorlinks=true]{hyperref}       
\usepackage{url}            
\usepackage{booktabs}       
\usepackage{amsfonts}       
\usepackage{nicefrac}       
\usepackage{microtype}      
\usepackage{xcolor}         
\usepackage{amsmath}
\usepackage{graphicx}
\usepackage{amsthm}
\usepackage{amssymb}

\title{Tuned Contrastive Learning}

\author{%
  Chaitanya Animesh
    \\
  UC San Diego\\
  \texttt{canimesh@ucsd.edu} \\
  \And
  Manmohan Chandraker \\
  UC San Diego \\
  \texttt{mkchandraker@ucsd.edu} \\
}

\begin{document}

\maketitle

\begin{abstract}
In recent times, contrastive learning based loss functions have become increasingly popular for visual self-supervised representation learning owing to their state-of-the-art (SOTA) performance. Most of the modern contrastive learning methods generalize only to one positive and multiple negatives per anchor. A recent state-of-the-art, supervised contrastive (SupCon) loss, extends self-supervised contrastive learning to supervised setting by generalizing to multiple positives and negatives in a batch and improves upon the cross-entropy loss. In this paper, we propose a novel contrastive loss function -- Tuned Contrastive Learning (TCL) loss, that generalizes to multiple positives and negatives in a batch and offers parameters to tune and improve the gradient responses from hard positives and hard negatives. We provide theoretical analysis of our loss function's gradient response and show mathematically how it is better than that of SupCon loss. We empirically compare our loss function with SupCon loss and cross-entropy loss in supervised setting on multiple classification-task datasets to show its effectiveness. We also show the stability of our loss function to a range of hyper-parameter settings. Unlike SupCon loss which is only applied to supervised setting, we show how to extend TCL to self-supervised setting and empirically compare it with various SOTA self-supervised learning methods. Hence, we show that TCL loss achieves performance on par with SOTA methods in both supervised and self-supervised settings.
\end{abstract}

\section{Introduction}

Paucity of labeled data limits the application of supervised learning to various visual learning tasks \cite{ARB}. As a result, unsupervised \cite{GAN_Goodfellow,Unsup_DimReduce,Unsup_FineGrained} and self-supervised based learning methods \cite{SimCLR,BarlowTwins,BYOL,Swav} have garnered a lot of attention and popularity for their ability to learn from vast unlabeled data. Such methods can be broadly classified into two categories: generative methods and discriminative methods. Generative methods \cite{GAN_Goodfellow, Unsup_FineGrained} train deep neural networks to generate in the input space i.e. the pixel space and hence, are computationally expensive and not necessary for representation learning. On the other hand, discriminative approaches  \cite{gidaris2018unsupervised, doersch2015unsupervised, zhang2016colorful, Oord_CPC, Bachman_MI, SimCLR} train deep neural networks to learn representations for pretext tasks using unlabeled data and an objective function. Out of these discriminative based approaches, contrastive learning based methods \cite{Oord_CPC, Bachman_MI, SimCLR} have performed significantly well and are an active area of research.

The common principle of contrastive learning based methods in an unsupervised setting is to create semantic preserving transformations of the input which are called \emph{positives} and treat transformations of other samples in a batch as \emph{negatives} \cite{ssl_cookbook,Supcon}. The contrastive loss objective considers every transformed sample as a reference sample, called an \emph{anchor}, and is then used to train the network architecture to pull the positives (for that anchor) closer to the anchor and push the negatives away from the anchor in latent space \cite{ssl_cookbook,Supcon}. The positives are often created using various data augmentation strategies. Supervised Contrastive Learning \cite{Supcon} extended contrastive learning to supervised setting by using the label information and treating the other samples in the batch having the same label as that of the anchor also as positives in addition to the ones produced through data augmentation strategies. It presents a new loss called supervised contrastive loss (abbreviated as SupCon loss) that can be viewed as a loss generalizing to multiple available positives in a batch.

\begin{figure}
  \centering
  \begin{tabular}{cc}
  \includegraphics[width=0.225\textwidth]{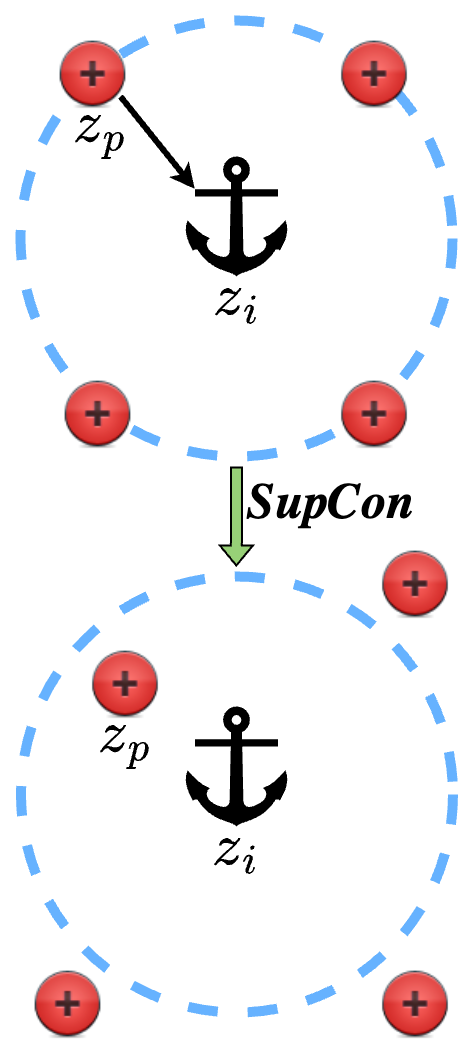} &\hspace{50pt}
  \includegraphics[width=0.225\textwidth]{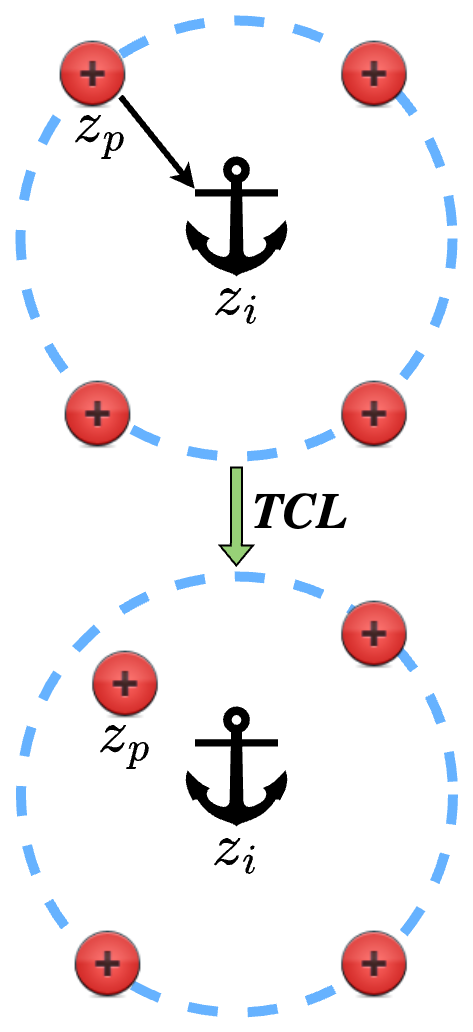}
  \end{tabular}
  \caption{Figure illustrates intuitively how TCL loss differs from SupCon loss \cite{Supcon}. For the SupCon loss per sample — $L_{i}^{sup}$ (from equation \ref{L_i_sup}) to decrease, the anchor $z_{i}$ will pull the positive $z_{p}$ but push away the other positives by some extent in the embedding space. TCL loss introduces parameters to reduce this effect and helps improve performance.}
\end{figure}

In this work, we propose a novel contrastive learning loss objective, which we call \textbf{Tuned Contrastive Learning (TCL) Loss} that can use multiple positives and multiple negatives present in a batch. We show how it can be used in supervised as well as self-supervised settings. TCL loss improves upon the limitations of the SupCon loss: 1. Implicit consideration of positives as negatives and, 2. No provision of regulating hard negative gradient response. TCL loss thus gives better gradient response to hard positives and hard negatives. This leads to small (< 1\% in terms of classification accuracy) but consistent improvements in performance over SupCon loss and comprehensive outperformance over cross-entropy loss. Since TCL generalizes to multiple positives, we then present a novel idea of having and using positive triplets (and possibly more) instead of being limited to positive pairs for self-supervised learning. We evaluate our loss function in self-supervised settings without making use of any label information and show how TCL outperforms SimCLR \cite{SimCLR} and performs on par with various SOTA self-supervised learning methods \cite{BYOL,BarlowTwins,ARB,Swav,MOCO,MOCOV2,w_mse,SimSiam,Vicreg}. Our key contributions in the paper are as follows:

\begin{enumerate}
    \item We identify and analyse in detail two limitations of the supervised contrastive (SupCon) loss.
    \item We present a novel contrastive loss function called Tuned Contrastive Learning (TCL) loss that generalizes to multiple positives and multiple negatives in a batch, overcomes the described limitations of the SupCon loss and is applicable in both supervised and self-supervised settings. We mathematically show with clear proofs how our loss's gradient response is better than that of SupCon loss.
    \item We compare TCL loss with SupCon loss (as well as cross-entropy loss) in supervised settings on various classification datasets and show that TCL loss gives consistent improvements in top-1 accuracy over SupCon loss. We empirically show the stability of TCL loss to a range of hyperparameters: network architecture, batch size, projector size and augmentation strategy.
    \item At last, we present a novel idea of having positive triplets (and possibly more) instead of positive pairs and show how TCL can be extended to self-supervised settings. We empirically show that TCL outperforms SimCLR, and performs on par with various SOTA self-supervised learning (SSL) methods. 
    
\end{enumerate}

\section{Related Work}

In this section, we cover various popular and recent works in brief involving contrastive learning.

Deep Metric learning methods originated with the idea of contrastive losses and were introduced with the goal of learning a distance metric between samples in a high-dimensional space \cite{ssl_cookbook}. The goal in such methods is to learn a function that maps similar samples to nearby points in this space, and dissimilar samples to distant points. There is often a margin parameter, $m$, imposing the distance between examples from different classes to be larger than this value of $m$ \cite{ssl_cookbook}. The triplet loss \cite{tripletloss} and the proposed improvements \cite{quadruplet_loss,lifted_structured_loss} on it used this principle. These methods rely heavily on sophisticated sampling techniques for choosing samples in every batch for better training.

SimCLR \cite{SimCLR}, an Info-NCE \cite{Oord_CPC} loss based framework, learns visual representations by increasing the similarity between the embeddings of two augmented views of the input image. Augmented views generally come from a series of transformations like  random resizing, cropping, color jittering, and random blurring. Although they make use of multiple negatives, only one positive is available per anchor. They require large batch sizes in order to have more hard negatives in the batch to learn from and boost the performance. SupCon loss \cite{Supcon} applies contrastive learning in supervised setting by basically extending the SimCLR loss to generalize to multiple positives available in a batch and improves upon the cross-entropy loss which lacks robustness to noisy labels \cite{ce_issue_noisy_1,ce_issue_noisy_2} and has the possibility of poor margins \cite{ce_issue_margin_1,ce_issue_margin_2}.

Unlike SimCLR or SupCon, many SOTA SSL approaches only work with positives (don't require negatives) or use different approach altogether. BYOL \cite{BYOL} uses asymmetric networks with one network using an additional predictor module while the other using exponential moving average (EMA) to update its weights, in order to learn using positive pairs only and prevent collapse. SimSiam \cite{SimSiam} uses stop-gradient operation instead of EMA and asymmetric networks to achieve the same goal. Barlow Twins \cite{BarlowTwins} objective function on the other hand computes the cross-correlation matrix between the embeddings of two identical networks fed with augmentations of a batch of samples, and tries to make this matrix close to identity. SwAV uses a clustering approach and enforces consistency between the cluster assignments of multiple positives produced through multi-crop strategy \cite{Swav}.

\section{Methodology}
\subsection{Supervised Contrastive Learning \& Its Issues}

The framework for Supervised Contrastive Learning consists of three components: a data augmentation module that produces two augmentations for each sample in the batch, an encoder network that maps the augmentations to their corresponding representation vectors and a projection network that produces normalized embeddings for the representation vectors to be fed to the loss function. The projection network is later discarded and the encoder network is used at inference time by training a linear classifier (attached to the frozen encoder) with cross-entropy loss. Section 3.1 of \cite{Supcon} contains more details on this.
The SupCon loss is given by the following two equations (refers to $L_{out}^{sup}$ in \cite{Supcon}):

\begin{equation}
    L^{sup}= \sum_{i \in I}L_{i}^{sup}
\end{equation}
 where
\begin{equation}
 L_{i}^{sup}=\frac{-1}{|P(i)|} \sum_{p \in P(i)} \text{log}(\frac{\text{exp}(z_{i}.z_{p}/\tau)}{\sum_{p' \in P(i)} \text{exp}(z_{i}.z_{p'}/\tau) + \sum_{n \in N(i)} \text{exp}(z_{i}.z_{n}/\tau)})
\label{L_i_sup}
 \end{equation}

Here $I$ denotes the batch of samples obtained after augmentation and so, will be twice the size of the original input batch. $i \in I$ denotes a sample (anchor) within it. $z_{i}$ denotes the normalized projection network embedding for the sample $i$ as given by the projector network. $P(i)$ is the set of all positives for the anchor $i$ (except the anchor $i$ itself) i.e. positive from the augmentation module and positives with the same label as anchor $i$ in the batch $I$. $N(i)$ denotes the set of negatives in the batch such that $N(i) \equiv I \setminus (P(i) \cup \{i\} )$. As shown in Section 2 of the supplementary material of \cite{Supcon}, we have the following lemma:
\newtheorem{lemma}{Lemma}
\begin{lemma}
The gradient of the SupCon loss per sample — $L_{i}^{sup}$ with respect to the normalized projection network embedding $z_{i}$ is given by:

\begin{equation}
    \frac{\partial L_{i}^{sup}}{\partial z_{i}} = \frac{1}{\tau} (\underbrace{\sum_{p \in P(i)} z_{p}(P_{ip}^{s}-X_{ip})}_\text{Gradient response from positives}  \hspace{12pt} + \underbrace{\sum_{n \in N(i)} z_{n}P_{in}^{s}}_\text{Gradient response from negatives})
\label{L_i_sup_grad}
\end{equation}

where 

\begin{equation}X_{ip}= \frac{1}{|P(i)|}
\end{equation}
\begin{equation}
P_{ip}^{s} = \frac{\text{exp}(z_{i}.z_{p}/\tau)}{\sum_{a \in A(i)} \text{exp}(z_{i}.z_{a}/\tau)}
\label{P_ip_s}
\end{equation}
\begin{equation}P_{in}^{s} = \frac{\text{exp}(z_{i}.z_{n}/\tau)}{\sum_{a \in A(i)} \text{exp}(z_{i}.z_{a}/\tau)}
\label{P_in_s}
\end{equation}

\end{lemma}




Note that $A(i) \equiv P(i) \cup N(i)$ here. The authors further show in Section 3 of the supplementary \cite{Supcon} that the gradient from a positive while flowing back through the projector into the encoder reduces to almost zero for easy positives and $|P_{ip}^{s}-X_{ip}|$ for a hard positive because of the normalization consideration in the projection network. Similarly, the gradient from a negative reduces to almost zero for easy negatives and $|P_{in}^{s}|$ for a hard negative. We now present and analyse the following two limitations of the SupCon loss:
\begin{enumerate}
    \item \textbf{Implicit consideration of positives as negatives}: Having a closer look at the $L_{i}^{sup}$ (equation \ref{L_i_sup}) loss term reveals that the numerator inside the log function considers similarity with one positive $p$ at a time while the denominator consists of similarity terms of the anchor $i$ with all the positives in the batch — the set $P(i)$, thereby implicitly considering all the positives as negatives. A glance at the derivation of Lemma 1 in \cite{Supcon} clearly shows that this leads to the magnitude of the gradient response from a hard positive getting reduced to $|X_{ip}-P_{ip}^{s}|$ instead of simply $|X_{ip}|$. The term $P_{ip}^{s}$ consists of an exponential term in the numerator and thus can reduce the magnitude of $|X_{ip}-P_{ip}^{s}|$ considerably, especially because the temperature $\tau$ is generally chosen to be small. Note that the authors of \cite{Supcon} approximate the numerator of $P_{ip}^{s}$ to 1 while considering the magnitude of $|X_{ip}-P_{ip}^{s}|$ in their supplementary by assuming $z_{i}.z_{p} \approx 0$ for a hard positive which might not always be true. Another way to look at this limitation analytically is to observe the log part in the $L_{i}^{sup}$ term. For the loss term to decrease and ideally converge to close to zero, the numerator term inside the log function will encourage the anchor $z_{i}$ to pull the positive $z_{p}$ towards it while the denominator term will encourage it to push away the other positives present in $P(i)$ by some extent, thereby treating the other positives as negatives implicitly.

    \item \textbf{No possibility of regulating $P_{in}^{s}$}: \cite{SimCLR, Supcon} mention that performance in contrastive learning benefits from hard negatives and gradient contribution from hard negatives should be higher. It is easy to observe from equation \ref{P_in_s} that the magnitude of the gradient signal from a hard negative — $|P_{in}^{s}|$ in the SupCon loss decreases with batch size and the number of positives in the batch, and can become considerably small, especially since the denominator consists of similarity terms between the anchor and all the positives in the batch which are temperature scaled and exponentiated. This can limit the gradient contribution from hard negatives.
\end{enumerate}

\subsection{Tuned Contrastive Learning}

In this section, we present our novel contrastive loss function — \textbf{Tuned Contrastive Learning (TCL) Loss}. Note that our representation learning framework remains the same as that of Supervised Contrastive Learning discussed above. The TCL loss is given by the following equations:
\begin{equation}
    L^{tcl}= \sum_{i \in I}L_{i}^{tcl}
\end{equation}
\begin{equation}
    L_{i}^{tcl}=\frac{-1}{|P(i)|} \sum_{p \in P(i)} \text{log}(\hspace{1pt} \frac{\text{exp}(z_{i}.z_{p}/\tau)}{D(z_{i})} \hspace{1pt})
\label{L_i_tcl}
\end{equation}
where
\begin{equation}
    D(z_{i}) = \sum_{p' \in P(i)} \text{exp}(z_{i}.z_{p'}/\tau)+ k_{1}(\sum_{p' \in P(i)} \text{exp}(-z_{i}.z_{p'}))+k_{2}(\sum_{n \in N(i)} \text{exp}(z_{i}.z_{n}/\tau))
\label{D_zi}
\end{equation}
\begin{equation}
k_{1},k_{2}\geq1
\end{equation}

$k_{1}$ and $k_{2}$ are scalar parameters that are fixed before training. All other symbols have the same meaning as discussed in the previous section. We now present the following lemma:

\begin{lemma}

The gradient of the TCL loss per sample — $L_{i}^{tcl}$ with respect to the normalized projection network embedding $z_{i}$ is given by:

\begin{equation}
    \frac{\partial L_{i}^{tcl}}{\partial z_{i}} = \frac{1}{\tau} (\underbrace{\sum_{p \in P(i)} z_{p}(P_{ip}^{t}-X_{ip}-Y_{ip}^{t})}_\text{Gradient response from positives} + \underbrace{\sum_{n \in N(i)} z_{n}P_{in}^{t}}_\text{Gradient response from negatives})
\label{L_i_tcl_grad}
\end{equation}

where 
\begin{equation}X_{ip}= \frac{1}{|P(i)|}
\end{equation}
\begin{equation}
P_{ip}^{t} = \frac{\text{exp}(z_{i}.z_{p}/\tau)}{D(z_{i})}
\label{P_ip_t}
\end{equation}
\begin{equation}
    Y_{ip}^{t}=\frac{\tau k_{1} \text{exp}(-z_{i}.z_{p})}{D(z_{i})}
\label{Y_ip_t}
\end{equation}
\begin{equation}P_{in}^{t} = \frac{k_{2}\hspace{0.5pt}\text{exp}(z_{i}.z_{n}/\tau)}{D(z_{i})}
\label{P_in_t}
\end{equation}

\end{lemma}

From Lemma 2, Theorem 1 and Theorem 2 follow in a straightforward fashion. The proofs for Lemma 2 and the two theorems are provided in our supplementary.

\newtheorem{theorem}{Theorem}
\begin{theorem}
For $k_{1},k_{2} \geq 1$, the magnitude of the gradient from a hard positive for TCL is strictly greater than the magnitude of the gradient from a hard positive for SupCon and hence, the following result follows:

\begin{equation}
     \underbrace{|X_{ip}-P_{ip}^{t}+Y_{ip}^{t}|}_\text{(TCL's hard positive gradient)} > \underbrace{|X_{ip}-P_{ip}^{s}|}_\text{(Supcon's hard positive gradient)}
\label{theorem1_eq}
\end{equation}

\end{theorem}

\begin{theorem}
For fixed $k_{1}$, the magnitude of the gradient response from a hard negative for TCL loss — $P_{in}^{t}$ strictly increases with $k_{2}$.
\end{theorem}

\paragraph{Effects of $k_{1}$ and $k_{2}$} The authors of SupCon show (in equation 18 in the supplementary of \cite{Supcon}) that the magnitude of gradient response from a hard positive $|X_{ip}-P_{ip}^{s}|$ increases with the number of positives and negatives in the batch. This is basically a result of reducing the value of $P_{ip}^{s}$, a term that results from having positive similarity terms in the denominator of $L_{i}^{sup}$. But they approximate the numerator of $P_{ip}^{s}$ to 1 by assuming $z_{i}.z_{p} \approx 0$ for a hard positive which might not always be true (especially since $\tau$ is typically chosen to be small like 0.1). As evident from the proof of Theorem 1 in our supplementary, we further push this idea and reduce the value of $P_{ip}^{s}$ in SupCon loss to $P_{ip}^{t}$ in TCL loss by having an extra term in the denominator involving $k_{1}$ — $k_{1}(\sum_{p' \in P(i)} \text{exp}(-z_{i}.z_{p'}))$ and choosing a large enough value for $k_{1}$. Hence, it reduces the effect of implicit consideration of positives as negatives, the first limitation of SupCon loss discussed in the previous section. Note that having the extra term to increase the gradient response from hard positive is not the same as increasing the gradient response by amplifying the learning rate. This is because for the same and fixed learning rate, TCL loss increases the magnitude of the gradient signal over SupCon loss by changing the coefficient of $z_{p}$ in equation \ref{L_i_tcl_grad} which in turn means changing the gradient direction as well. This leads to consistently better performance as shown in the numerous experiments that we perform. Also, it directly follows from Theorem 2 that $k_{2}$ allows to regulate (increase) the gradient signal from a hard negative and thus, overcomes the second limitation of the SupCon loss.

\paragraph{Augmentation Strategy for Self-Supervised Setting} Since TCL loss can use multiple positives, we consider working with positive triplets instead of positive pairs in self-supervised settings. Given a batch $B$ with $N$ samples, we produce augmented batch $I$ of size $3N$ by producing three augmented views (positives) for each sample in $B$. This idea can further be extended in different ways to have more positives per anchor. For example, one can think of combining different augmentation strategies to produce multiple views per sample although we limit ourselves positive triplets in this work.
\section{Experiments}

We evaluate TCL in three stages: 1. Supervised setting, 2. Hyper-parameter stability and 3. Self-supervised setting. We then present empirical analysis on TCL loss's parameters — $k_{1}$ and $k_{2}$ and show how we choose their values. All the relevant training details are mentioned in our supplementary.

\subsection{Supervised Setting}
We start by evaluating TCL in supervised setting first. Since the authors of \cite{Supcon} mention that SupCon loss performs significantly better than triplet loss \cite{tripletloss} and N-pair loss \cite{NPair_Loss}, we directly compare TCL loss with SupCon and cross-entropy losses on various classification benchmarks including CIFAR-10, CIFAR-100 \cite{cifar}, Fashion MNIST (FMNIST) \cite{fmnist} and ImageNet-100 \cite{ImageNet}. The encoder network chosen is ResNet-50 \cite{resnet} for CIFAR and FMNIST datasets while Resnet-18 \cite{resnet} for ImageNet dataset (because of memory constraints). The representation vector is the activation of the final pooling layer of the encoder. ResNet-18 and ResNet-34 encoders give 512 dimensional representation vectors while ResNet-50 and above produce 2048 dimensional vectors. The projector network is a MLP with one hidden layer with size being 512 for ResNet-18 and Resnet-34, and 2048 for ResNet-50 and higher networks. The output layer of the projector MLP is 128 dimensional for all the networks. We use the same cross-entropy implementation as used by Supervised Contrastive Learning \cite{Supcon}.

\textbf{Note that for fair comparison of TCL with Supervised Contrastive Learning, we keep the architecture and all other possible hyper-parameters except the learning rate exactly the same. We also do hyper-parameter tuning significantly more for Supervised Contrastive Learning than for TCL}. As a result, we found that our re-implementation of Supervised Contrastive Learning gave better results than what is reported in the paper \cite{Supcon}. For example, on CIFAR-100 our significantly tuned version of SupCon achieves 79.1\% top-1 classification accuracy, 2.6\% more than what is reported in SupCon paper. As the authors of SupCon \cite{Supcon} mention that 200 epochs of contrastive training is sufficient for training a ResNet-50 on complete ImageNet dataset, our observations for the supervised setting case on relatively smaller datasets like CIFAR, FMNIST and ImageNet-100 are consistent with this finding. We train Resnet-50 (and ResNet-18) for a total of 150 epochs – 100 epochs of contrastive training for the encoder and the projector followed by 50 epochs of cross-entropy training for the linear layer. Note that 150 epochs of total training was sufficient for our re-implementation of SupCon loss to achieve better results than reported in the paper (2.6\% more on CIFAR-100 and 0.3\% more on CIFAR-10). We anyways still provide results for 250 epochs of training in our supplementary. As Table \ref{sup_table} shows, TCL loss consistently performs better than SupCon loss and outperforms cross-entropy loss on all the datasets. 

 \begin{table}[h!]
  \caption{Comparisons of top-1 accuracies of TCL with SupCon and cross-entropy loss in supervised settings. The values in parenthesis for SupCon denote the values presented in their paper.}
  \label{sup_table}
  \centering
  \begin{tabular}{lllll}
    \toprule
    Dataset & Cross-Entropy     & SupCon     & TCL \\
    \midrule
    CIFAR-10 & 95.0 & 96.3 (96.0) & 96.4     \\
    CIFAR-100 & 75.3 & 79.1 (76.5) & 79.8     \\
    FashionMNIST &  94.5&   95.5 &      95.7\\
    ImageNet-100 &  84.2 &   85.9 &  86.7   \\
    \bottomrule
  \end{tabular}
\end{table}

\subsection{Hyper-parameter Stability}

\begin{figure}
  \centering
  \begin{tabular}{cc}
  \includegraphics[width=6.5cm, height=5cm]{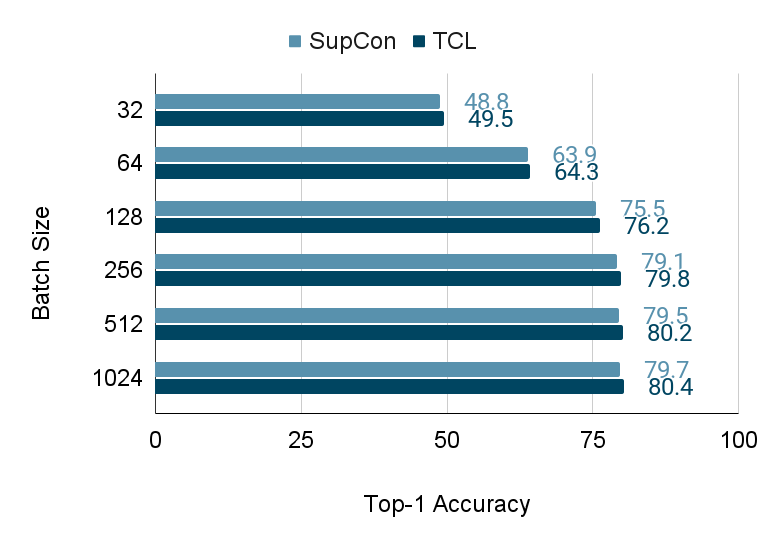} &
  \includegraphics[width=6.5cm, height=5cm]{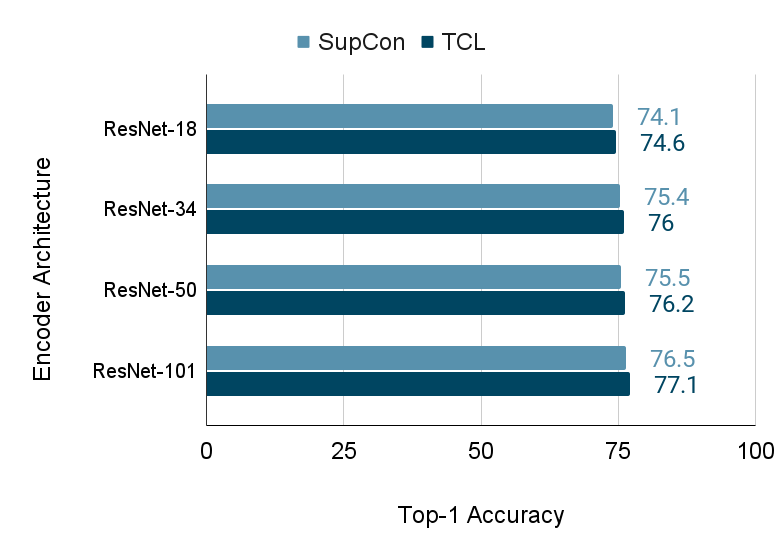}\\
  \includegraphics[width=6.5cm, height=5cm]{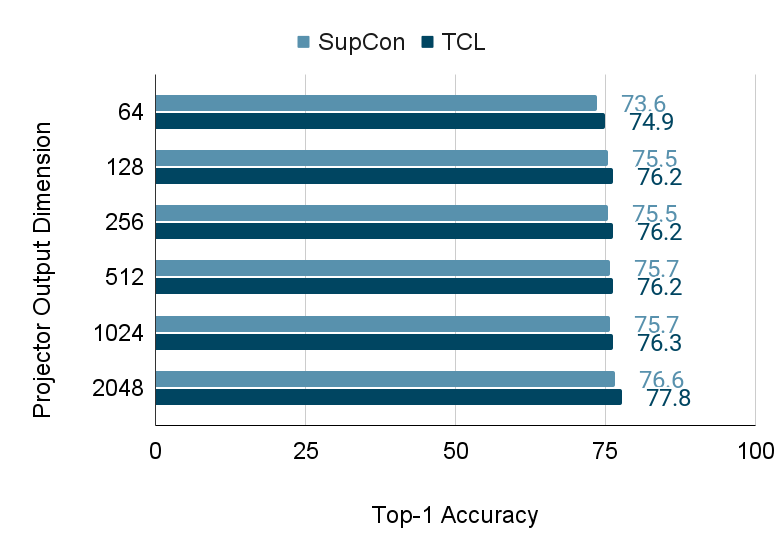} &
  \includegraphics[width=6.5cm, height=5cm]{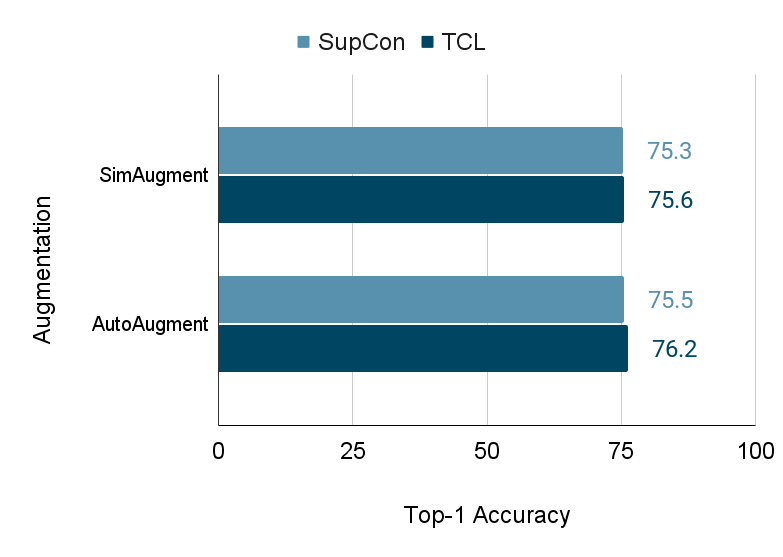}
  \end{tabular}
  \caption{SupCon vs TCL losses on a range of hyper-parameters. (a). batch size (top left) (b). encoder architecture (top right) (c). projector output dimensions/size (bottom left) (d). augmentation method (bottom right)}
  \label{fig_hparam}
\end{figure}

We now show the stability of TCL loss to a range of hyper-parameters. We compare TCL loss with SupCon loss on various hyper-parameters — encoder architectures, batch sizes, projection embedding sizes and different augmentations. For all the hyper-parameter experiments we choose CIFAR-100 as the common dataset (unless stated otherwise), set total training epochs to 150 (same as earlier section), temperature $\tau$ to 0.1 and use SGD optimizer with momentum=0.9 and weight decay=$1e-4$.

\subsubsection{Encoder Architecture}

We choose 4 encoder architectures of varying sizes- ResNet-18, ResNet-34, ResNet-50 and ResNet-101. For both TCL loss and SupCon loss, we choose batch size as 128 and AutoAugment \cite{autoaugment} data augmentation method. As evident from Fig. \ref{fig_hparam}-(b), TCL loss achieves consistent improvements in top-1 test classification accuracy over SupCon loss on all the architectures. We also tested TCL loss and SupCon loss on ImageNet-100 with ResNet-18 (batch size of 256) and ResNet-34 (batch size of 128). Using ResNet-18, TCL loss achieved 86.7\% top-1 accuracy while SupCon loss achieved 85.9\% top-1 accuracy. By switching to ResNet-34, TCL loss got 87.2\% top-1 accuracy while SupCon loss got 86.5\% top-1 accuracy.


\subsubsection{Batch Size}

For comparing TCL loss with SupCon loss on different batch sizes, we choose ResNet-50 as the encoder architecture and AutoAugment \cite{autoaugment} data augmentation. As evident from Fig. \ref{fig_hparam}-(a), we observe that TCL loss consistently performs better than SupCon loss on all batch sizes. All the batch sizes mentioned are after performing augmentation. Note that the authors of SupCon loss use an effective batch size of 256 (after augmentation) for CIFAR datasets in their released code\footnote{https://github.com/HobbitLong/SupContrast}. We select batch sizes equal to, smaller and greater than this value for comparison to demonstrate the effectiveness of Tuned Contrastive Learning.


\subsubsection{Projection Network Embedding (\texorpdfstring{$z_{i}$}{Lg}) Size}

In this section we analyse empirically how SupCon and TCL losses perform on various projection network output embedding sizes. This particular experiment was not explored as stated by the authors of Supervised Contrastive Learning \cite{Supcon}. ResNet-50 is the common encoder used with Auto-Augment \cite{autoaugment} data augmentation. As evident from Fig. \ref{fig_hparam}-(c), we observe that TCL loss achieves consistent improvements in top-1 test classification accuracy over SupCon loss for various projector output sizes. We observe that 64 performs the worst while 128, 256, 512 and 1024 give similar results. 2048 performs the best for both with TCL loss achieving 1.2\% higher accuracy than SupCon loss for this size.


\subsubsection{Augmentations}
We choose two augmentation strategies — AutoAugment and SimAugment for comparisons. AutoAugment\cite{autoaugment} is a two-stage augmentation policy trained with reinforcement learning and gives stronger (aggressive and diverse) augmentations. SimAugment \cite{SimCLR} is relatively a weaker augmentation strategy used in SimCLR that applies simple transformations like random flips, rotations, color jitters and gaussian blurring. We don't use gaussian blur in our implementation of SimAugment and train for 100 extra epochs i.e. 250 epochs while using it. Fig. \ref{fig_hparam}-(d) shows that TCL loss performs better than SupCon loss with both augmentations although, the gain is more with AutoAugment – the stronger augmentation strategy.

\subsection{Self-Supervised Setting}
In this section we evaluate TCL without any labels in self-supervised setting by making use of positive triplets as described earlier. We compare TCL with various SOTA SSL methods as shown in Table \ref{ssl_table}. The results for these methods are taken from the works of \cite{ARB}, \cite{solo}. The datasets used for comparison are CIFAR 10, CIFAR-100 and ImageNet-100. ResNet-18 is the common encoder used for every method. For CIFAR-10 and CIFAR-100 every method uses 1000 epochs of contrastive pre-training including TCL. For ImageNet-100, every method does 400 epochs of contrastive pre-training.
\begin{table}[h]
  \caption{Comparison of top-1 accuracy of TCL with various SSL methods. Values in bold show the best performing method.}
  \label{ssl_table}
  \centering
  \begin{tabular}{lllll}
    \toprule
    {Method} & {Projector Size} & CIFAR-10 & CIFAR-100     & ImageNet-100 \\
    \midrule
    BYOL\cite{BYOL} & 4096 &  92.6&   \textbf{70.2}& \textbf{80.1}      \\
    DINO\cite{dino} & 256 &  89.2&  66.4&   74.8\\
    SimSiam\cite{SimSiam} & 2048 &  90.5&  65.9 & 77.0     \\
    MOCO V2\cite{MOCO,MOCOV2} & 256 &  \textbf{92.9}&   69.5&   78.2   \\
    ReSSL\cite{ressl} & 256 &  90.6&   65.8&     76.6\\
    VICReg\cite{Vicreg} & 2048 &  90.1&   68.5&  79.2   \\
    SwAV\cite{Swav} & 256 &  89.2&   64.7&    74.3 \\
    W-MSE\cite{w_mse} & 256 &  88.2&   61.3&   69.1  \\
    ARB\cite{ARB} & 256 &  91.8&   68.2&     74.9\\
    ARB\cite{ARB} & 2048 &  92.2&   69.6&     79.5\\
    Barlow-Twins\cite{BarlowTwins} & 256 &  87.4&   57.9& 67.2    \\
    Barlow-Twins\cite{BarlowTwins} & 2048 &  89.6&   69.2& 78.6    \\
    SimCLR\cite{SimCLR} & 256 &  90.7&   65.5&   77.5  \\
    \textbf{TCL (Self-Supervised)} & 256 &  91.6&   66.7&  77.9   \\
    \bottomrule
    \textbf{TCL (Supervised)} & 128 &  95.8&  77.5&  86.7   \\
    \bottomrule
  \end{tabular}
\end{table}


Table \ref{ssl_table} shows the top-1 accuracy achieved by various methods on the three datasets. TCL performs consistently better than SimCLR \cite{SimCLR} and performs on par with various other methods. Note that methods like BYOL \cite{BYOL}, VICReg \cite{Vicreg}, ARB \cite{ARB} and Barlow-Twins \cite{BarlowTwins} use much larger projector size for output embedding and extra hidden layers in the projector MLP to get better performance while MOCO V2 \cite{MOCOV2} uses a queue size of 32,768 to get better results. Few of the methods like BYOL \cite{BYOL}, SimSiam \cite{SimSiam}, MOCO V2 \cite{MOCO,MOCOV2} also maintain two networks and hence, effectively use double the number of parameters and are memory intensive.

We also add the results of supervised TCL that can make use of labels as it is generalizable to any number of positives. Supervised TCL achieves significantly better results than all other SSL methods. SwAV does use a multi-crop strategy to create multiple augmentations but is not extended to supervised setting to use the labels \cite{Swav}.

\subsection{Analyzing and Choosing \texorpdfstring{$k_{1}$}{Lg} and \texorpdfstring{$k_{2}$}{Lg} for TCL}

\begin{figure}
  \centering
  \begin{tabular}{cc}
  \includegraphics[width=6.5cm, height=3.75cm, keepaspectratio, clip]{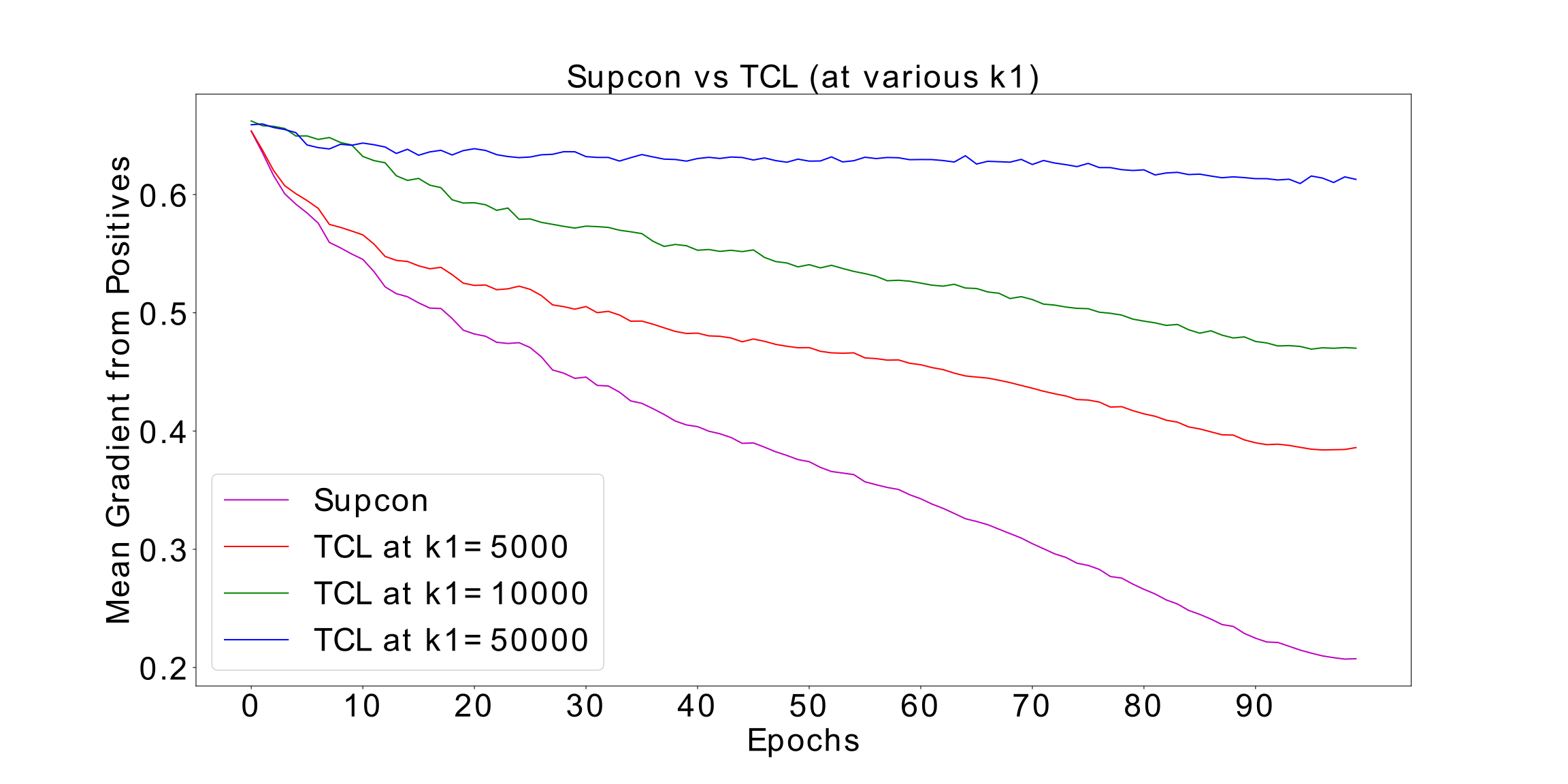} &
  \includegraphics[width=6.5cm, height=3.75cm, keepaspectratio, clip]{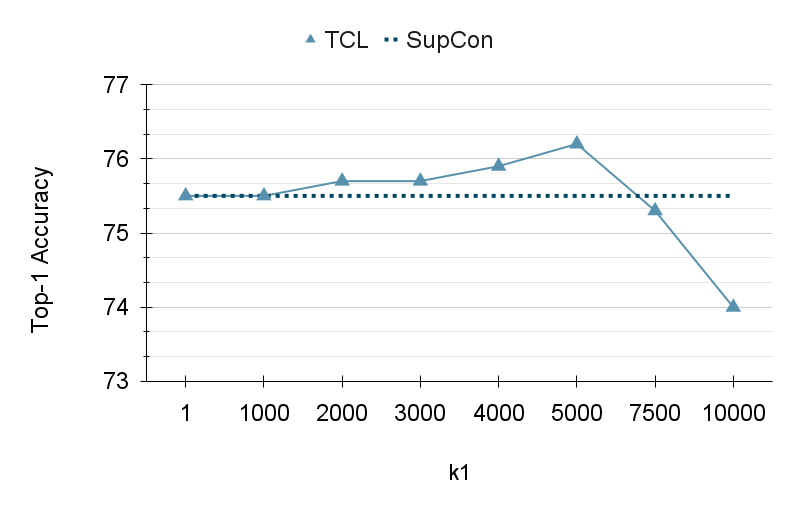}\\
  \includegraphics[width=6.5cm, height=3.75cm,keepaspectratio, clip]{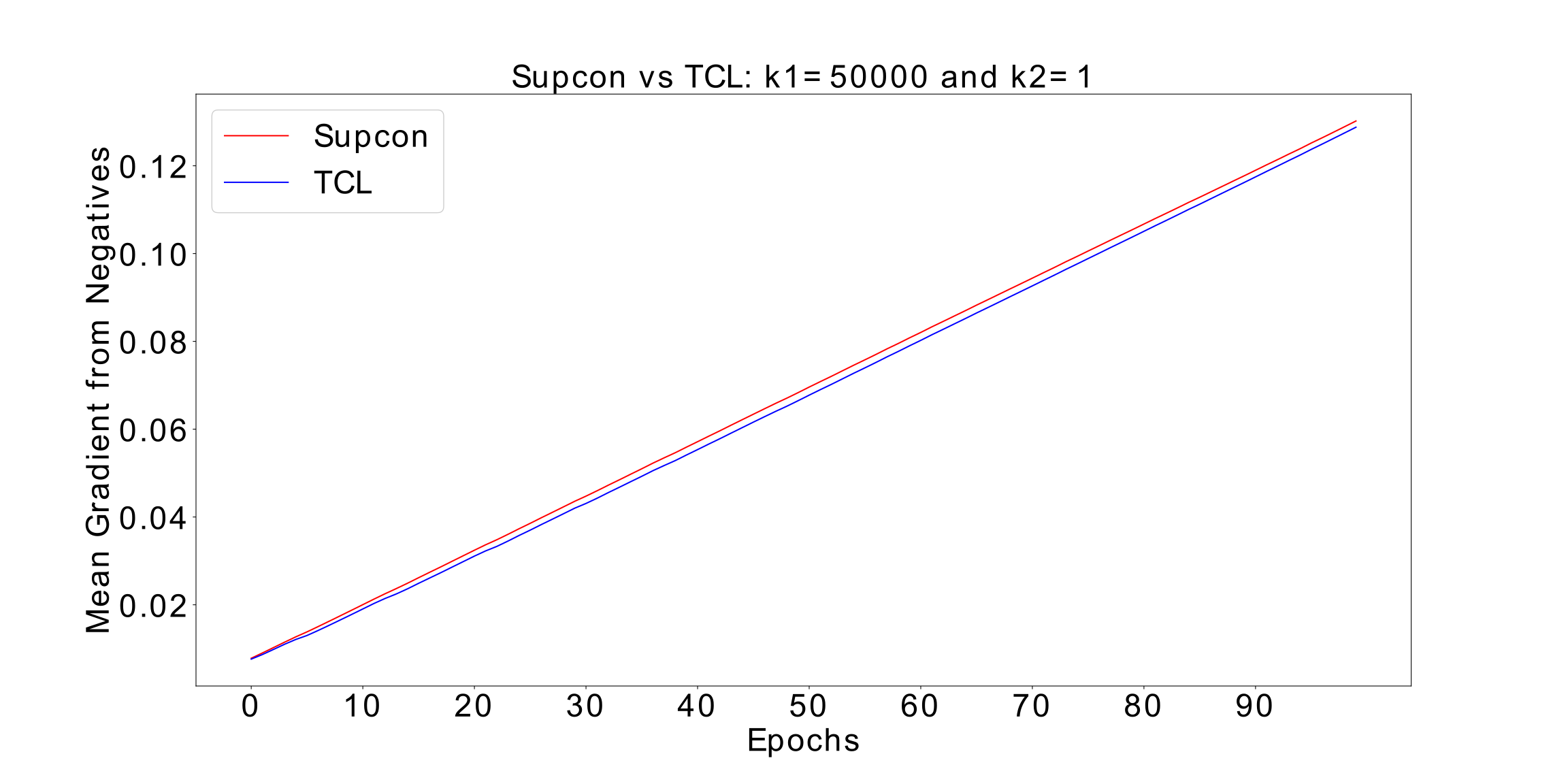} &
  \includegraphics[width=6.5cm, height=3.75cm, keepaspectratio, clip]{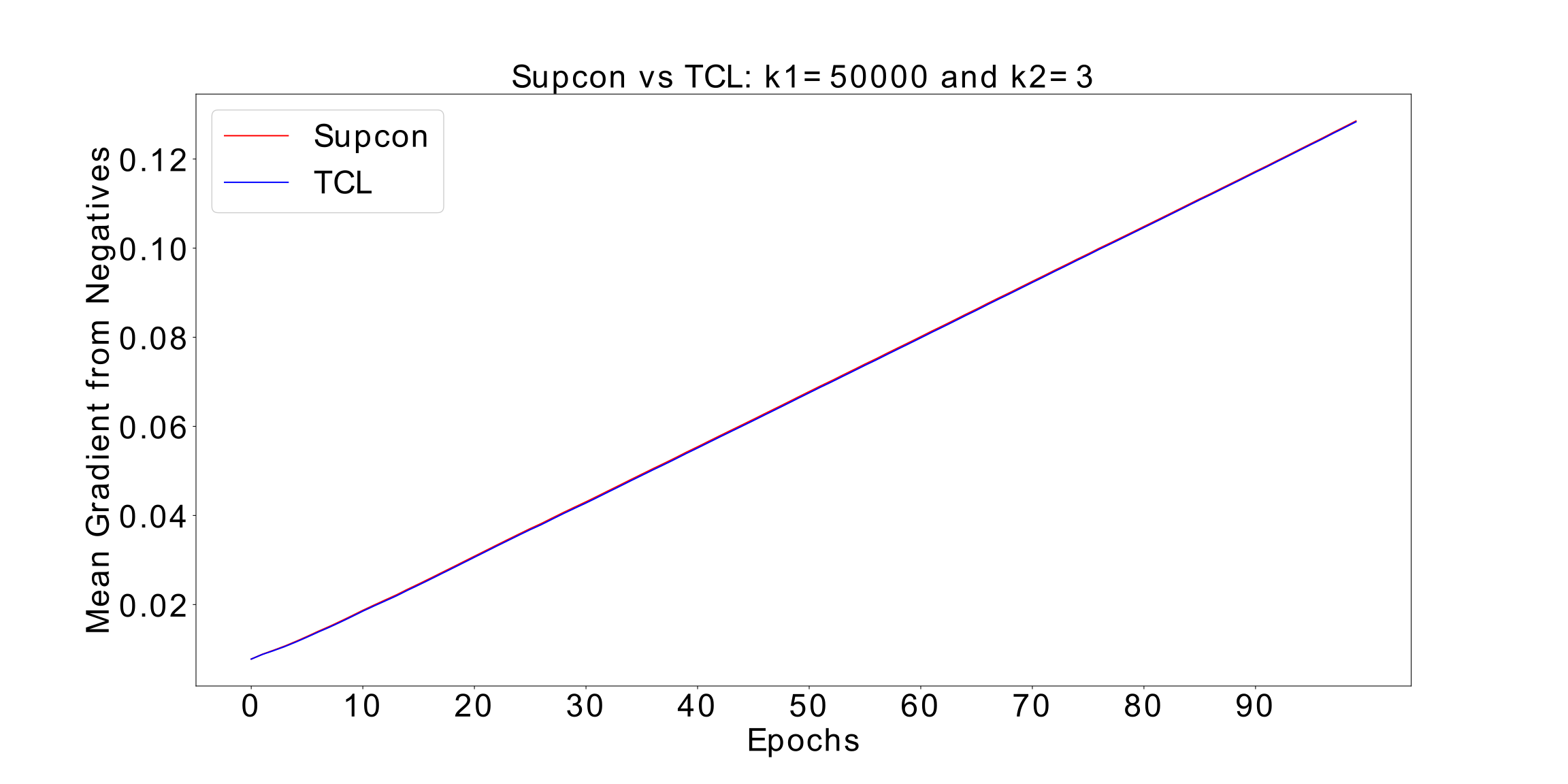}
  \end{tabular}
  \caption{Analysis of $k_{1}$ and $k_{2}$ (a). plot of mean gradient from positives for SupCon and TCL (at various values of $k_{1}$) (top left) (b). top-1 accuracy vs  $k_{1}$ on CIFAR-100 (top right) (c). plot of mean gradient from negatives for SupCon and TCL ($k_{1}=50000$ and $k_{2}=1$) (bottom left) (d). plot of mean gradient from negatives for SupCon and TCL ($k_{1}=50000$ and $k_{2}=3$) (bottom right)}
  \label{fig_k1k2}
\end{figure}
As we discussed earlier in Section 3.2, $k_{1}$ helps in increasing the magnitude of positive gradient from positives while $k_{2}$ helps in regulating (increasing) the gradient from negatives. We verify our claims empirically and show how we go about choosing their values for training.

\paragraph{Analyzing effects of $k_{1}$}We calculate the mean gradient from all positives (expressions from equation \ref{theorem1_eq}) per anchor averaged across the batch and plot the values for SupCon loss and TCL loss over the course of training of ResNet-50 on CIFAR-100 for 100 epochs. As evident from  Fig. \ref{fig_k1k2}-(a), increasing the value of $k_{1}$ increases the magnitude of gradient response from positives. We also analyze how this correlates with the top-1 accuracy in Fig. \ref{fig_k1k2}-(b). As we see for small values of $k_{1}$, the top-1 accuracy remains more or less the same as that of SupCon loss. As we increase it further, the gradient from positives increase leading to gains in top-1 accuracy. The top-1 accuracy reaches a peak and then starts to drop with further increase in $k_{1}$. We hypothesize that this drop is because very large values of $k_{1}$ start affecting the gradient response from negatives (equations \ref{P_in_t} and \ref{D_zi}). We verify this hypothesis while analyzing $k_{2}$.

\paragraph{Analyzing effects of $k_{2}$}We calculate the mean gradient from all negatives (expressions from equations \ref{P_in_s} and \ref{P_in_t}) per anchor averaged across the batch for the same setting as above and plot the values for SupCon loss and TCL loss. As we see in Fig. \ref{fig_k1k2}-(c), TCL loss's gradient lags behind SupCon loss's gradient by some margin for $k_{1}=50000$ and $k_{2}=1$. This value of $k_{1}$ actually leads to a top-1 accuracy of 71.8\%, a drop in performance. When we start increasing the value of $k_{2}$, the gradient response from negatives increase for TCL loss. Fig. \ref{fig_k1k2}-(d) shows that by increasing $k_{2}$ to 3 while $k_{1}=50000$, the gap between gradient (from negatives) curves of TCL loss and SupCon loss vanishes. We also observe that the top-1 accuracy increases back to 76.2\%, the best possible accuracy that we got for this setting.

\paragraph{Choosing $k_{1}$ and $k_{2}$}We observe that a value of $k_{1}$ in the range of $10^{3}$ to $10^{4}$ works the best with $k_{1}=$ $4 \times 10^{3}$ or $5 \times 10^{3}$ almost always working on all datasets and configurations we experimented with. We generally start with these two values or otherwise with $2 \times 10^{3}$ and increase it in steps of 2000 till $8\times10^{3}$. We also observed during our experiments that choosing any value less than $5 \times 10^{3}$ always gave improvements in performance over SupCon loss. For most of our experiments we set $k_{1}$ to $4 \times 10^{3}$ or $5 \times 10^{3}$ and get the desired performance boost in a single run. We found $k_{2}$ to be useful to compensate for the reduction in the value of $P_{in}^{t}$ caused by increasing $k_{1}$ and especially in self-supervised settings where hard negative gradient contribution is important. For setting $k_{2}$, we fix $k_{1}$ (which itself gives boost in performance) and increase $k_{2}$ in steps of 0.1 or 0.2 to see if we can get further improvement. We provide values for $k_{1}$ and $k_{2}$ for all our experiments in the supplementary. As we see, we generally keep $k_{2}=1$ for supervised settings but we do sometimes set it to a value slightly bigger than 1. We set $k_{2}$ to a higher value in self-supervised settings as compared to supervised settings to get higher gradient contribution from hard negatives. Increasing $k_{1}$ didn't help much in boosting the performance in self-supervised setting (as we only had two positives per anchor) and so we set it to 1. Increasing $k_{2}$ also increases the gradient response from positives to some extent by decreasing $P_{ip}^{t}$ (equation \ref{P_ip_t}) and so, we found it sufficient to increase only $k_{2}$ and set $k_{1}$ to 1 in self-supervised setting.

\section{Conclusion \& Limitations}
In this work, we have presented a novel contrastive loss function called Tuned Contrastive Learning (TCL) loss that generalizes to multiple positives and multiple negatives present in a batch and is applicable to both supervised and self-supervised settings. We showed mathematically how its gradient response to hard negatives and hard-positives is better than that of SupCon loss. We evaluated TCL loss in supervised and self-supervised settings and showed that it performs on par with existing state-of-the-art supervised and self-supervised learning methods. We also showed empirically the stability of TCL loss to a range of hyper-parameter settings.

A limitation of our work is that the proposed loss objective introduces two extra parameters $k_{1}$ and $k_{2}$, for which the values are chosen heuristically. Future direction can include works that come up with loss objectives that provide the properties of TCL loss out of the box without introducing any extra parameters.

\bibliography{arxiv_2023}
\bibliographystyle{plain}
\appendix

\section{Proofs for Theoretical Results}

\paragraph{Proof for Lemma 1:} Section 2 of the supplementary material of SupCon \cite{Supcon} gives a clear proof for Lemma 1 (refer to the derivation of $L_{out}^{sup}$ in that section).

\setcounter{lemma}{1}
\begin{lemma}
The gradient of the TCL loss per sample — $L_{i}^{tcl}$ with respect to the normalized projection network embedding $z_{i}$ is given by:

\begin{equation}
    \frac{\partial L_{i}^{tcl}}{\partial z_{i}} = \frac{1}{\tau} (\underbrace{\sum_{p \in P(i)} z_{p}(P_{ip}^{t}-X_{ip}-Y_{ip}^{t})}_\text{Gradient response from positives} + \underbrace{\sum_{n \in N(i)} z_{n}P_{in}^{t}}_\text{Gradient response from negatives})
\end{equation}

where 
\begin{equation}X_{ip}= \frac{1}{|P(i)|}
\end{equation}
\begin{equation}
P_{ip}^{t} = \frac{\text{exp}(z_{i}.z_{p}/\tau)}{D(z_{i})}
\end{equation}
\begin{equation}
    Y_{ip}^{t}=\frac{\tau k_{1} \text{exp}(-z_{i}.z_{p})}{D(z_{i})}
\end{equation}
\begin{equation}P_{in}^{t} = \frac{k_{2}\hspace{0.5pt}\text{exp}(z_{i}.z_{n}/\tau)}{D(z_{i})}\end{equation}

\end{lemma}

\begin{proof}
\begin{equation}
    L_{i}^{tcl}=\frac{-1}{|P(i)|} \sum_{p \in P(i)} \text{log}(\hspace{1pt} \frac{\text{exp}(z_{i}.z_{p}/\tau)}{D(z_{i})} \hspace{1pt})
\end{equation}
\begin{equation}
    \implies L_{i}^{tcl}=\frac{-1}{|P(i)|} \sum_{p \in P(i)}  (\frac{z_{i}.z_{p}}{\tau} - \text{log}(D(z_{i})) 
\end{equation}

\begin{equation}
\begin{aligned}
    \implies \frac{\partial L_{i}^{tcl}}{\partial z_{i}}=\frac{-1}{\tau |P(i)|} \sum_{p \in P(i)}  \bigg( z_{p} - \frac{(\sum_{p' \in P(i)} z_{p'}\text{exp}(z_{i}.z_{p'}/\tau)}{D(z_{i})}\\
    + \frac{\tau k_{1}(\sum_{p' \in P(i)} z_{p'}\text{exp}(-z_{i}.z_{p'}))}{D(z_{i})}- \frac{ k_{2}(\sum_{n \in N(i)} z_{n}\text{exp}(z_{i}.z_{n}/\tau))}{D(z_{i})} \bigg)
\end{aligned}
\end{equation}

\begin{equation}
\begin{aligned}
    \implies \frac{\partial L_{i}^{tcl}}{\partial z_{i}}=\frac{-1}{\tau |P(i)|} \Bigg[ \sum_{p \in P(i)} z_{p}
    - \sum_{p \in P(i)}\frac{(\sum_{p' \in P(i)} z_{p'}\text{exp}(z_{i}.z_{p'}/\tau))}{D(z_{i})}\\
    + \sum_{p \in P(i)}\frac{\tau k_{1}(\sum_{p' \in P(i)} z_{p'}\text{exp}(-z_{i}.z_{p'}))}{D(z_{i})}
    - \sum_{p \in P(i)} \frac{ k_{2}(\sum_{n \in N(i)} z_{n}\text{exp}(z_{i}.z_{n}/\tau))}{D(z_{i})} \Bigg]
\end{aligned}
\end{equation}

\begin{equation}
\begin{aligned}
    \implies \frac{\partial L_{i}^{tcl}}{\partial z_{i}}=\frac{-1}{\tau |P(i)|} \Bigg[ \sum_{p \in P(i)}  z_{p}
    - \sum_{p' \in P(i)}\frac{(\sum_{p \in P(i)} z_{p'}\text{exp}(z_{i}.z_{p'}/\tau))}{D(z_{i})}\\
    + \sum_{p' \in P(i)}\frac{\tau k_{1}(\sum_{p \in P(i)} z_{p'}\text{exp}(-z_{i}.z_{p'}))}{D(z_{i})}
    - \sum_{p \in P(i)} \frac{ k_{2}(\sum_{n \in N(i)} z_{n}\text{exp}(z_{i}.z_{n}/\tau))}{D(z_{i})} \Bigg]
\end{aligned}
\end{equation}

\begin{equation}
\begin{aligned}
    \implies \frac{\partial L_{i}^{tcl}}{\partial z_{i}}=\frac{-1}{\tau |P(i)|} \Bigg[ \sum_{p \in P(i)}   z_{p}
    - \sum_{p' \in P(i)}\frac{(|P(i)| z_{p'}\text{exp}(z_{i}.z_{p'}/\tau))}{D(z_{i})}\\
    + \sum_{p' \in P(i)}\frac{\tau k_{1}(|P(i)| z_{p'}\text{exp}(-z_{i}.z_{p'}))}{D(z_{i})}
    - \frac{ |P(i)| k_{2}(\sum_{n \in N(i)} z_{n}\text{exp}(z_{i}.z_{n}/\tau))}{D(z_{i})} \Bigg]
\end{aligned}
\end{equation}

\begin{equation}
\begin{aligned}
    \implies \frac{\partial L_{i}^{tcl}}{\partial z_{i}}=\frac{-1}{\tau |P(i)|} \Bigg[ \sum_{p \in P(i)}  z_{p}
    - \sum_{p \in P(i)}\frac{(|P(i)| z_{p}\text{exp}(z_{i}.z_{p}/\tau))}{D(z_{i})}\\
    + \sum_{p \in P(i)}\frac{\tau k_{1}(|P(i)| z_{p}\text{exp}(-z_{i}.z_{p}))}{D(z_{i})}
    - \frac{ |P(i)| k_{2}(\sum_{n \in N(i)} z_{n}\text{exp}(z_{i}.z_{n}/\tau))}{D(z_{i})} \Bigg]
\end{aligned}
\end{equation}

\begin{equation}
\begin{aligned}
    \implies \frac{\partial L_{i}^{tcl}}{\partial z_{i}}=\frac{-1}{\tau} \Bigg[ \sum_{p \in P(i)}  \frac{z_{p}}{|P(i)|}
    - \sum_{p \in P(i)}\frac{(z_{p}\text{exp}(z_{i}.z_{p}/\tau))}{D(z_{i})}\\
    + \sum_{p \in P(i)}\frac{\tau k_{1}(z_{p}\text{exp}(-z_{i}.z_{p}))}{D(z_{i})}
    - \frac{ k_{2}(\sum_{n \in N(i)} z_{n}\text{exp}(z_{i}.z_{n}/\tau))}{D(z_{i})} \Bigg]
\end{aligned}
\end{equation}

\begin{equation}
\begin{aligned}
    \implies \frac{\partial L_{i}^{tcl}}{\partial z_{i}}=\frac{1}{\tau} \Bigg[ \sum_{p \in P(i)}  z_{p} \bigg(
    \frac{\text{exp}(z_{i}.z_{p}/\tau)}{D(z_{i})}
    -\frac{1}{|P(i)|}
    - \frac{\tau k_{1}\text{exp}(-z_{i}.z_{p})}{D(z_{i})} \bigg)\\
    + \sum_{n \in N(i)} z_{n}  \frac{ k_{2} \text{exp}(z_{i}.z_{n}/\tau)}{D(z_{i})} \Bigg]
\end{aligned}
\end{equation}
This completes the proof.
\end{proof}

\setcounter{theorem}{0}
\begin{theorem}
For $k_{1},k_{2} \geq 1$, the magnitude of the gradient from a hard positive for TCL is strictly greater than the magnitude of the gradient from a hard positive for SupCon and hence, the following result follows:

\begin{equation}
     \underbrace{|X_{ip}-P_{ip}^{t}+Y_{ip}^{t}|}_\text{(TCL's hard positive gradient)} > \underbrace{|X_{ip}-P_{ip}^{s}|}_\text{(Supcon's hard positive gradient)}
\end{equation}

\end{theorem}

\begin{proof}
As the authors of \cite{Supcon} show in Section 3 of their supplementary (we also mention the same in our main paper in Section 3.1) that the gradient from a positive while flowing back through the projector into the encoder reduces to almost zero for easy positives and $|P_{ip}^{s}-X_{ip}|$ for a hard positive because of the normalization consideration in the projection network combined with the assumption that $z_{i}.z_{p} \approx 1$ for easy positives and $z_{i}.z_{p} \approx 0$ for hard positives. Proceeding in a similar manner, it is straightforward to see that the gradient response from a hard positive in case of TCL is $|P_{ip}^{t}-X_{ip}-Y_{ip}^{t}|$. We don't prove this explicitly again since the derivation will be identical to what authors \cite{Supcon} have already shown. One can refer section 3 of the supplementary of \cite{Supcon} for details.\\

Now, because $k_{1},k_{2} \geq 1$, it is easy to observe from equations 5 and 13 of our main paper that,

\begin{equation}
P_{ip}^{t} < P_{ip}^{s}
\end{equation}
And from equation 14 of our main paper:
\begin{equation}
     Y_{ip}^{t} > 0 
\end{equation}

Hence, the result follows. This completes the proof.
\end{proof}

\begin{theorem}
For fixed $k_{1}$, the magnitude of the gradient response from a hard negative for TCL — $P_{in}^{t}$ increases strictly with $k_{2}$.
\end{theorem}

\begin{proof}
\begin{equation}
P_{in}^{t} = \frac{k_{2}\hspace{0.5pt}\text{exp}(z_{i}.z_{n}/\tau)}{D(z_{i})}
\end{equation}

\begin{equation}
= \frac{k_{2}\hspace{0.5pt}\text{exp}(z_{i}.z_{n}/\tau)}{\sum_{p' \in P(i)} \text{exp}(z_{i}.z_{p'}/\tau)+ k_{1}(\sum_{p' \in P(i)} \text{exp}(-z_{i}.z_{p'}))+k_{2}(\sum_{n \in N(i)} \text{exp}(z_{i}.z_{n}/\tau))}
\end{equation}

\begin{equation}
= \frac{\hspace{0.5pt}\text{exp}(z_{i}.z_{n}/\tau)}{\big(\sum_{p' \in P(i)} \text{exp}(z_{i}.z_{p'}/\tau)+ k_{1}(\sum_{p' \in P(i)} \text{exp}(-z_{i}.z_{p'})) \big)/k_{2} +(\sum_{n \in N(i)} \text{exp}(z_{i}.z_{n}/\tau))}
\end{equation}

It is now easy to observe that for a fixed $k_{1}$, $P_{in}^{t}$ increases strictly with $k_{2}$. This completes the proof.
\end{proof}

\section{Training Details}

\subsection{Supervised Setting}
We first present the common training details used for each dataset experiment in the supervised setting for SupCon \cite{Supcon} and TCL. Except for the contrastive training learning rate, every other detail presented is common for SupCon and TCL. As mentioned in our main paper, we train for a total of 150 epochs which involves 100 epochs of contrastive training for the encoder and the projector, and 50 epochs of cross-entropy training for the linear layer for both the losses. AutoAugment \cite{autoaugment} is the common data augmentation method used except for FMNIST \cite{fmnist} for which we used a simple augmentation strategy consisting of random cropping and horizontal flip. We use cosine annealing based learning rate scheduler and SGD optimizer with momentum=0.9 and weight decay=$1e-4$ for both contrastive and linear layer training. Temperature $\tau$ is set to 0.1. For linear layer training, the starting learning rate is $5e-1$. ResNet-50 \cite{resnet} is the common encoder architecture used. We use NVIDIA-GeForce-RTX-2080-Ti, NVIDIA-TITAN-RTX and NVIDIA-A100-SXM4-80GB GPUs for our experiments.

\paragraph{CIFAR-10 \cite{cifar}} Image size is resized to $32\times32$ in the data augmentation pipeline. We use a batch size of 128. For both SupCon and TCL we use a starting learning rate of $1e-1$ for contrastive training. We set $k_{1}=5000$ and $k_{2}=1$ for TCL.

\paragraph{CIFAR-100 \cite{cifar}} Image size is resized to $32\times32$ in the data augmentation pipeline. We use a batch size of 256. For both SupCon and TCL we use a starting learning rate of $2e-1$ for contrastive training. We set $k_{1}=4000$ and $k_{2}=1$ for TCL.

\paragraph{FMNIST \cite{fmnist}} Image size is resized to $28\times28$ in the data augmentation pipeline. We use a batch size of 128. For both SupCon and TCL we use a starting learning rate of $9e-2$ for contrastive training. We set $k_{1}=5000$ and $k_{2}=1$ for TCL.

\paragraph{ImageNet-100 \cite{ImageNet}} Images are resized to $224\times224$ in the data-augmentation pipeline and batch size of 256 is used. For SupCon we use a starting learning rate of $2e-1$ for contrastive training while $3e-1$ for TCL. We set $k_{1}=4000$ and $k_{2}=1$ for TCL.

For CIFAR-100 dataset and batch size of 128, we also ran the experiment 30 times to get 95\% confidence intervals for top-1 accuracies of SupCon and TCL. For SupCon we got $74.79 \pm 0.23$ while for TCL we got $75.72 \pm 0.16$ as the confidence intervals. We also present results for 250 epochs of training constituted by 200 epochs of contrastive training and 50 epochs of linear layer training in Table \ref{sup_table_250}. As we see, TCL performs consistently better than SupCon \cite{Supcon}. Note that we didn't see any performance improvement for FMNIST dataset for either SupCon or TCL by running them for 250 epochs.

 \begin{table}[h!]
  \caption{Comparisons of top-1 accuracies of TCL with SupCon in supervised setting for 250 epochs of training.}
  \label{sup_table_250}
  \centering
  \begin{tabular}{lllll}
    \toprule
    Dataset & SupCon     & TCL \\
    \midrule
    CIFAR-10 &  96.7&  96.8    \\
    CIFAR-100 & 81.0 & 81.6     \\
    FashionMNIST &  95.5   & 95.7    \\
    ImageNet-100 &   86.5 &  87.1   \\
    \bottomrule
  \end{tabular}
\end{table}

\subsection{Hyper-parameter Stability}

For the hyper-parameter stability experiments we have presented most of the details in the main paper. We present the learning rates and values of $k_{1}$ and $k_{2}$ used for TCL. Remaining details are the same as the supervised setting experiments.

\subsubsection{Encoder Architecture}

The starting learning rate for contrastive training is $1e-1$ for all the encoders except ResNet-101 for which we used a value of $9e-2$. $k_{1}=5000$ and $k_{2}=1$ are the common values used for all the encoders.

\subsubsection{Batch Size}

For batch sizes=32, 64, 128, 256, 512 and 1024 we set the starting learning rates for contrastive training to $8e-3$, $9e-3$, $1e-1$, $2e-1$, $5e-1$ and $1$ respectively. For batch size of 32 we used $k_{1}=5000$ and $k_{2}=1$. For batch size of 64 we used $k_{1}=7500$ and $k_{2}=1$. For batch size of 128 we used $k_{1}=5000$ and $k_{2}=1$. For batch sizes of 256, 512 and 1024 we used $k_{1}=4000$ and $k_{2}=1$.

\subsubsection{Projection Network Embedding (\texorpdfstring{$z_{i}$}{Lg}) Size}

We used a common starting learning rate of $1e-1$ with $k_{1}=5000$ and $k_{2}=1$ for all the projector output sizes.

\subsubsection{Augmentations}
For AutoAugment \cite{autoaugment} method, we use a learning rate of $1e-1$ with $k_{1}=5000$ and $k_{2}=1$. For SimAugment \cite{SimCLR}, we use a learning rate of $1e-1$ with $k_{1}=5000$ and $k_{2}=1.2$.

\subsection{Self-Supervised Setting}

For the self-supervised setting, we reuse the code provided by \cite{solo} and we are thankful to them for providing all the required details. The projector used for TCL is exactly the same as SimCLR for fair comparison and consists of one hidden layer of size 2048 and output size of 256. ResNet-18 is the common encoder used for all the methods. We use SGD optimizer with momentum=0.9 wrapped with LARS optimizer \cite{lars} and weight deacy of $1e-4$. Augmentation used is SimAugment \cite{SimCLR} and is done in the same manner as \cite{solo}. Gaussian blur is used for self-supervised setting. We use NVIDIA-GeForce-RTX-2080-Ti, NVIDIA-TITAN-RTX and NVIDIA-A100-SXM4-80GB GPUs for our experiments.

\paragraph{CIFAR-10 \cite{cifar}} All methods do 1000 epochs of contrastive pre-training on CIFAR-10 and images are reshaped to $32\times32$ in the data augmentation pipeline. We use batch size=256, same as SimCLR. For TCL, we use a starting learning rate of $4e-1$ for contrastive pre-training with $k_{1}=1$ and $k_{2}=1.5$. 

\paragraph{CIFAR-100 \cite{cifar}} All methods do 1000 epochs of contrastive pre-training on CIFAR-100 and images are reshaped to $32\times32$ in the data augmentation pipeline. We use batch size=256, same as SimCLR. For TCL, we use a starting learning rate of $4e-1$ for contrastive pre-training with $k_{1}=1$ and $k_{2}=1.5$.

\paragraph{ImageNet-100 \cite{ImageNet}} All methods do 400 epochs of contrastive pre-training on ImageNet-100 and images are rescaled to a size of $224\times224$. We use batch size=256, same as used by SimCLR. For TCL, we use a starting learning rate of $4e-1$ for contrastive pre-training with $k_{1}=1$ and $k_{2}=1.5$.

\end{document}